\documentclass[letterpaper]{article} 
\usepackage{aaai2026}  
\usepackage{times}  
\usepackage{helvet}  
\usepackage{courier}  
\usepackage[hyphens]{url}  
\usepackage{graphicx} 
\usepackage{natbib}  
\usepackage{caption} 
\usepackage{booktabs}
\usepackage[switch]{lineno}

\usepackage{todonotes}
\usepackage{amsmath, amsthm, amsxtra, amsfonts, amssymb, amstext, mathtools}
\usepackage{thmtools} 
\usepackage{nicefrac}
\usepackage{xspace}
\usepackage{xcolor}
\usepackage{enumitem}
\usepackage[algo2e,ruled,vlined,linesnumbered]{algorithm2e}\SetArgSty{upshape}
\usepackage[nameinlink]{cleveref}

\usepackage[shortcuts]{extdash} 

\newtheorem{theorem}{Theorem}
\newtheorem{lemma}[theorem]{Lemma}

\newcommand{\NSGA}{\mbox{NSGA}\nobreakdash-II\xspace}
\newcommand{\NSGAthree}{\mbox{NSGA}\nobreakdash-III\xspace}
\newcommand{\SPEA}{\mbox{SPEA2}\xspace}
\newcommand{\SMS}{\mbox{SMS\nobreakdash-EMOA}\xspace}

\newcommand{\OMM}{\textsc{OMM}\xspace}

\newcommand{\oneminmax}{\textsc{OneMinMax}\xspace}

\newcommand{\lotz}{\textsc{LOTZ}\xspace}
\newcommand{\leadingonestrailingzeros}{\textsc{LeadingOnes\-TrailingZeros}\xspace}
\newcommand{\ojzj}{\textsc{OneJumpZeroJump}\xspace}
\newcommand{\ojzjk}{\textsc{OJZJ$_k$}\xspace}

\DeclareMathOperator{\mutate}{mutate}
\DeclareMathOperator{\eliminate}{eliminate}
\DeclareMathOperator{\adddominated}{add \_ dominated}

\DeclareMathOperator{\E}{E}

\newcommand{\R}{\ensuremath{\mathbb{R}}}

\newcommand{\Z}{\ensuremath{\mathbb{Z}}}

\DeclareMathOperator{\Bin}{Bin}

\let\originalleft\left
\let\originalright\right
\renewcommand{\left}{\mathopen{}\mathclose\bgroup\originalleft}
\renewcommand{\right}{\aftergroup\egroup\originalright}

\DeclarePairedDelimiter{\floor}{\lfloor}{\rfloor}
\DeclarePairedDelimiter{\ceil}{\lceil}{\rceil}
\DeclarePairedDelimiterXPP{\ones}[1]{}{\lvert}{\rvert}{_1}{#1}
\DeclarePairedDelimiterXPP{\zeros}[1]{}{\lvert}{\rvert}{_0}{#1}

\newcommand{\abs}[1]{\lvert #1 \rvert} 
\newcommand{\norm}[1]{\left\lVert#1\right\rVert_2}

\title{Improved Runtime Guarantees for the \SPEA Multi-Objective Optimizer}
\author{
  Benjamin Doerr\textsuperscript{\rm 1},
  Martin~S. Krejca\textsuperscript{\rm 1},
  Milan Stanković\textsuperscript{\rm 2}
}
\affiliations{
  \textsuperscript{\rm 1}Laboratoire d'Informatique (LIX), CNRS, École Polytechnique, Institut Polytechnique de Paris\\
  \textsuperscript{\rm 2}École Polytechnique, Institut Polytechnique de Paris\\
  \{first-name.last-name\}@polytechnique.edu
}

\begin{document}

\maketitle

\sloppy{
\begin{abstract}
    Together with the NSGA-II, the SPEA2 is one of the most widely used domination-based multi-objective evolutionary algorithms. For both algorithms, the known runtime guarantees are linear in the population size; for the NSGA-II, matching lower bounds exist. With a careful study of the more complex selection mechanism of the SPEA2, we show that it has very different population dynamics. From these, we prove runtime guarantees for the \oneminmax, \leadingonestrailingzeros, and \ojzj benchmarks that depend less on the population size. For example, we show that the SPEA2 with parent population size $\mu \ge n - 2k + 3$ and offspring population size $\lambda$ computes the Pareto front of the \ojzj benchmark with gap size~$k$ in an expected number of $O( (\lambda+\mu)n + n^{k+1})$ function evaluations. This shows that the best runtime guarantee of $O(n^{k+1})$ is not only achieved for $\mu = \Theta(n)$ and $\lambda = O(n)$ but for arbitrary $\mu, \lambda = O(n^k)$. Thus, choosing suitable parameters~-- a key challenge in using heuristic algorithms~-- is much easier for the SPEA2 than the NSGA-II.
\end{abstract}

\section{Introduction}

Many real-world problems are faced with optimizing several conflicting objectives at the same time.
This commonly results in a plethora of incomparable optimal trade-offs (the \emph{Pareto optima}), collectively known as the \emph{Pareto front}.
Owing to the incomparable nature of the Pareto optima, it is desirable to have optimization algorithms propose as many of them as possible~-- ideally the entire front~-- in as little time as possible.
Multi-objective evolutionary algorithms (MOEAs) lend themselves naturally to this kind of task, as they iteratively maintain a set of best-so-far solutions, and they are known to be among the best and most popular approaches for solving multi-objective optimization problems~\cite{CoelloLV07,ZhouQLZSZ11}.

Over the last years, the strong performance of modern MOEAs in applications has been steadily complemented by theoretical results, see, e.g., \cite{LiZZZ16,ZhengLD22,DoNNS23,WiethegerD23,BianZLQ23,ZhengD24,DangOS24,Opris25aaai}.
These works provide rigorous bounds on the expected runtime of MOEAs until they find the entire Pareto front for the first time. Both the results and the arguments used in their proofs provide deeper insights into the working principles -- see \cite{DoerrQ23LB} for an analysis of the population dynamics of the NSGA-II --, allow to understand their different advantages and shortcomings -- see \cite{AlghouassDKL25} for a comparison of the approximation ability of NSGA-II and SPEA2 --, and can help design superior algorithm variants, see, e.g., \cite{BianZLQ23}.

In this work, we continue the mathematical study of the \emph{strength Pareto evolutionary algorithm~$2$} (\SPEA) \cite{ZitzlerLT01}. Together with the \NSGA~\cite{DebPAM02}, \NSGAthree \cite{DebJ14}, and \SMS \cite{BeumeNE07}, it is one of the widely used domination-based MOEAs. Due to its complex selection mechanism, it is the one least understood from the theoretical perspective. In the first runtime analysis of the \SPEA, \citeauthor{RenBLQ24}~(\citeyear{RenBLQ24}) show that with parent population size~$\mu$ at least the size of the Pareto front and offspring population size $\lambda$ at most $O(\mu)$, the \SPEA computes the Pareto front of the (bi-objective) \oneminmax, \leadingonestrailingzeros, and \ojzj benchmarks in an expected number of $O(\mu n \log n)$, $O(\mu n^2)$, and $O(\mu n^k)$ function evaluation, respectively. These bounds essentially agree with the runtime guarantees given earlier for the \NSGA \cite{ZhengD23aij,DoerrQ23tec} and \SMS \cite{BianZLQ23,ZhengD24}.\footnote{We note that \citeauthor{RenBLQ24}~(\citeyear{RenBLQ24}) actually showed bounds for arbitrary numbers $m$ of objectives. These were slightly improved in \cite{WiethegerD24} when $m$ is large. Since we only regard the bi-objective case, we do not discuss these results in detail.} The only other runtime analysis of the \SPEA \cite{AlghouassDKL25} compares the \NSGA and \SPEA in terms of their approximation ability, that is, how well they compute approximations to the Pareto front when the population size smaller than the size of the Pareto front (and only approximations can be computed).

\textbf{Our contribution:}
We take a closer look at the more complex selection mechanisms of the \SPEA, which -- different from the \NSGA -- takes into account the distances to \emph{all} other individuals on the Pareto front rather than only the two closest ones in each objective. We show that this leads to substantially different population dynamics. While in the \NSGA, large populations have a tendency to concentrate in the middle of the Pareto front \cite{DoerrQ23LB}, the selection of the \SPEA distributes multiple objective values evenly on the Pareto front (\Cref{lem:survival-guarantee}) for arbitrary multi-objective optimization problems.

This selection property combined with the fact that the \SPEA creates copies of individuals with moderate probability allows us to prove that in relatively short time, a given objective value is present in the population with the fair multiplicity (\Cref{lem:mult-growth2}). We prove this result for binary representations and bit-wise mutation, but it is clear that comparable results hold for other setting as long as copies of individuals are created sufficiently easily, and that our results also hold when the \SPEA applies crossover in each iteration with an at most constant probability different from~$1$.

Knowing that the multiplicities of objective values in the population are evenly distributed, we prove runtime guarantees that suffer less from larger population sizes. Recall from above that previous results for the \SPEA (and the same holds for all known results of the standard \NSGA and \SMS) depend linearly on the population size. Hence, the best performance guarantee for the \SPEA is obtained from a population size equal to the size of the Pareto front; any increase of the population size increases (that is, weakens) the performance guarantee in a linear fashion. In the bounds we prove in this work, this linear increase kicks in only for larger population sizes. This suggest that in a practical application where the size of the Pareto front is not known in advance, the choice of the population size of the \SPEA is less critical than of the \NSGA or \SMS\footnote{No lower bounds are proven for the \SMS, but from studying the proofs in \cite{DoerrQ23LB}, we expect to observe the same behavior as for the \NSGA.}.

We note that the selection problem of the \NSGA was recently \cite{DoerrIK25} overcome by adding a third selection criterion (after non-dominated sorting and crowding distance). This led to runtime guarantees comparable to the ones we prove here. Given that it is a purely theoretical work, there is no information yet on how the modification proposed in \cite{DoerrIK25} influences the performance of the \NSGA in practice. As observed in \cite{DoerrIK25}, it led to a moderate increase of the time to execute one iteration of the algorithm. In this perspective, it is interesting that our results shows that the standard \SPEA performs a balanced selection (leading to improved runtime guarantees) without any modification.

\paragraph{Our runtime results.}
We regard the three common  benchmarks \oneminmax, \leadingonestrailingzeros, and \ojzj of problem size~$n$. Our target is that the population witnesses the Pareto front. Consequently, we always assume that~$\mu$ is at least the size of the Pareto front, which is at most $n+1$ for our benchmarks. Different from \cite{RenQBQ24}, we allow arbitrary offspring population sizes $\lambda$, that is, we omit their assumption $\lambda = O(\mu)$.

For \oneminmax, one of the easiest benchmarks, we prove that after an expected number of $O((\mu + \lambda) n + n^2 \log n)$ function evaluations, the population witnesses the Pareto front (\Cref{thm:omm-runtime-bound}). Compared to the previous best bound $O(\mu n \log n)$ (for $\lambda = O(\mu)$ only), we observe an asymptotic runtime gain as soon as $\mu$ does not have the smallest possible value $\mu = \Theta(n)$.
As \oneminmax has a Pareto front of size $n + 1$, this result shows that the \SPEA  for all $n+1 \le \mu = O(n \log n)$ and all $\lambda = O(n \log n)$ displays an expected runtime of $O(n^2 \log n)$, which is the one observed with most standard MOEAs.
Moreover, this bound is strictly better than the expected runtime of $\Theta(\mu n \log n)$ known for the \NSGA \cite{ZhengD23aij,DoerrQ23LB} once $\mu = \omega(n)$.

For \leadingonestrailingzeros, we prove an expected runtime of $O((\mu + \lambda) n \log \frac{\mu}{n+1} + n^3 + \lambda n)$ function evaluations (\Cref{thm:lotz-runtime}), again when $\mu$ is at least the size $n + 1$ of the Pareto front. This is $O(n^3)$, the typical runtime of standard MOEAs on this benchmark, for all $n + 1 \leq \mu = O(n^2)$ and all $\lambda = O(n^2)$.

For \ojzj problems with gap size $k \in [2, {\frac{n}{2}}] \cap \Z$, the improvement is even more drastic.
We prove an expected runtime of $O((\mu+\lambda) n + n^{k+1})$ function evaluations (\Cref{thm:ojzj-runtime}), for $\mu \geq n - 2k + 3$.
That is, the typical runtime of $O(n^{k + 1})$ is achieved for the large range of all $n - 2k + 3 \le \mu = O(n^k)$ and all $\lambda = O(n^k)$.

In summary, our results show that the \SPEA attains the runtime which many known MOEAs achieve only with asymptotically optimal parameter values for much larger ranges of values. While naturally the precise bounds we prove depend on the optimization problem, they key arguments leading to these bounds, the balanced selection of the \SPEA and the multiplicative growth of the number of individuals having a certain objective value, hold under very general assumptions. For this reason, it appears justified to claim that the \SPEA is much more robust to changes of its parameters than the main other domination-based MOEAs.

\section{Preliminaries}
\label{sec:preliminaries}

We denote by $\Z$ the set of integers. For $m,n \in \Z$, $m\leq n$, we define $[m..n] \coloneqq [m,n] \cap \Z$.

For $n \in \Z_{\geq 1}$ and a bit string $x = (x_i)_{i \in [1 .. n]} \in \{0,1\}^n$, we define $\ones{x} \coloneqq  \sum_{i=1}^n x_i$ as the number of ones of $x$, and $\zeros{x} \coloneqq n - \ones{x}$ as the number of zeros of $x$. We define $1^n \coloneqq (1)_{i \in [1 .. n]}$ as the bit string containing only ones, and $0^n \coloneqq (0)_{i \in [1 .. n]}$ as the bit string containing only zeros.

\subsection{Multi-Objective Optimization}
\label{sec:multi-objective-opt}
Let $\Omega$ be some set, and let $m \in \Z_{\geq 2}$. Consider a function
\begin{align*}
    f \colon \Omega \rightarrow \R^m,
    x               \mapsto (f_1(x), \dots ,f_m(x)).
\end{align*}
We consider \emph{maximizing} $f$ with respect to the \emph{dominance} partial order, defined as follows. Let $x,y \in \Omega$. If for all $i \in [1..m] $ we have $f_i(x) \geq f_i(y)$, then we say $x$ \emph{weakly dominates} $y$ and write $x \succeq y$. Furthermore,  if $x \succeq y $ and $f(x) \neq f(y)$, then we say $x$ \emph{strictly dominates} $y$ and write $x \succ y$.
In all other cases, we say $x$ and $y$ are \emph{incomparable}.

If a point $x \in \Omega$ is not strictly dominated by any other point in $ \Omega$, then $x$ is \emph{Pareto-optimal}. The set $S^* = \{x \in \Omega \mid x \text{ is Pareto-optimal } \}$ of Pareto-optimal points is called the \emph{Pareto set}, and the set of their objective values, $F^* = \{f(x) \mid x \in S^* \}$, is called the \emph{Pareto front}.
We aim at finding a subset of the Pareto set that \emph{covers} the Pareto front, that is, finding a set $A \subseteq S^*$ such that $\{f(x) \mid x \in A \} = F^*$.

\subsection{The \SPEA}
\label{sec:spea-alg}

\Cref{alg:SPEA2} shows the generic \SPEA maximizing a multi-objective function $f\colon \Omega \rightarrow \R^m$, similar to the pseudocode provided by \cite{WiethegerD24}, who use updated notation (\Cref{tab:notation-comparison}).

We focus on the case where $\Omega$ is the set of bit strings of length $n \in \Z_{\geq 1}$, i.e., $ \Omega = \{0,1 \} ^n$, known as \emph{pseudo-Boolean maximization}. We call the sets $P_t$ and $Q_t$ in \Cref{alg:SPEA2} the \emph{parent} and the \emph{offspring population}, respectively. We emphasize that these are multi-sets, i.e, repetitions of elements are allowed. The set operations $\cup$ and $\setminus$ are the multi-set union and the multi-set difference. We also call the multi-set $R_t = P_t \cup Q_t $, i.e., the combined parent and offspring population, the \emph{total population} in iteration $t \in \Z_{\geq 0}$.

In case $\Omega = \{0,1\}^n$, the \textbf{mutate} function applies the standard bit mutation to a given bit string. That is, for all $x \in \{0,1\}^n$, a call to $\mutate(x)$ flips each bit of a copy of the bit string~$x$ independently with probability $\frac{1}{n}$, and returns the resulting bit string.

The function \textbf{eliminate} removes one individual from the given multi-set, according to a concept called \emph{$\sigma$-criterion}. The $\sigma$-criterion favors individuals that are further away from their neighbors, in the objective space. Formally, let~$A$ be a multi-set of elements from $\Omega$, and let $f(A)$ be the multi-set $ \{ f(x) \mid x \in A \}$. For $x \in A$ and $k \in [1..\abs{A} - 1]$, we define $\sigma_k(x)$ as the distance between $f(x)$ and its $k$-th nearest neighbor in $f(A)$. That is, we sort the list $( \norm{f(x) - f(y)})_{y \in A \setminus \{x\}}$ in increasing order, and define $\sigma_k(x)$ as the $k$-th member of the sorted list.
Furthermore, we define the relation~$\leq_d$ as the lexicographic order over $\{(\sigma_k(z))_{k \in [1..\abs{A}-1]}\}_{z \in A}$. Hence, for all $x, y \in A$, we have $x \leq_d y$ if and only if one of the following holds.
\begin{enumerate}
    \item For all $k \in [1.. \abs{A}-1]$, we have $\sigma_k(x) = \sigma_k(y)$.
    \item There exists a $k \in [1.. \abs{A} - 1]$ such that $\sigma_k(x) < \sigma_k(y)$, and, for all $\ell \in [1..k-1] $, we have $\sigma_{\ell}(x) = \sigma_{\ell}(y)$.
\end{enumerate}
A call to $\eliminate (A)$ removes an element $x$ from $A$ which is minimal with respect to $\leq_d$.
If there are multiple such elements, we break the tie uniformly at random.

We note that the $\sigma$-criterion initially requires $O(\mu^2)$ computations of shortest distances among all pairs of individuals, and it requires $O(\mu)$ shortest-distance computations afterward for each individual added or removed.
This can be expensive when compared to selection criterions from other MOEAs, such as the \NSGA.
However, when considering that a population does not change by that much from one iteration to the next and by keeping track of the number of copies of objective values (which eliminates the need to re-compute all known shortest distances), the overall cost of this operation can be reduced.
Since we focus on the number of function evaluations, for which a single evaluation is usually already very costly, the exact cost of the $\sigma$-criterion is not relevant to our analysis but remains an important aspect to keep in mind when interpreting our results.

The \textbf{add\textunderscore dominated} function returns an element from a given set according to some assignment of objective values. If, in iteration $t$, the number of non-dominated individuals in the total population $R_t$ is less than the parent population size $\mu$, the algorithm uses this function to select dominated individuals from the total population $R_t$ and include them in the new parent population $P_{t+1}$.
Note that the multi-set $R_t \setminus P_{t+1}$, which the $\adddominated$ function is applied to, contains only the individuals that are dominated by another individual from $R_t$.
However, in our analysis, we mainly focus on the process of generating new non-dominated individuals from the existing non-dominated individuals. Hence, the exact formulation of the $\adddominated$ function is not essential for the analysis. We refer the reader to the original \SPEA paper for a detailed description of this method \cite{ZitzlerLT01}.

\begin{algorithm2e}[t]
    \caption{The \emph{strength Pareto evolutionary algorithm~$2$}~\cite{ZitzlerLT01} (\SPEA) with parent population size $\mu \in \Z_{\geq 1}$ and offspring population size $\lambda \in \Z_{\geq 1}$, using standard bit mutation, maximizing an $m$-objective function $f\colon \Omega \rightarrow \R^m$.}
    \label{alg:SPEA2}

    $P_0 \leftarrow$ population of $\mu$ individuals sampled independently, uniformly at random from $\Omega$\;
    \BlankLine
    \For{$t \in \Z_{\geq 0}$}{
        \If{termination criterion is met}{
            return non-dominated individuals from $P_{t}$\;
        }
        $Q_t \leftarrow \emptyset$\;
        \For{$i \in [1 .. \lambda]$}{
            Select $x$ from $P_t$ uniformly at random\;
            $Q_t \leftarrow Q_t \cup \{ \mutate (x) \} $\;
        }
        $R_t \leftarrow P_t \cup Q_t$\;
        $S_t \leftarrow $ non-dominated individuals from $R_t$\;
        \While{$\abs{S_{t}} > \mu$}{
            $\eliminate(S_{t})$\;
        }
        \While{$\abs{S_{t}} < \mu$}{
            $S_{t} \leftarrow S_{t} \cup \{ \adddominated(R_t \setminus S_{t})$\}\;
        }
        $P_{t+1} \leftarrow S_t$\;
    }
\end{algorithm2e}

\paragraph*{Original \SPEA}
Although \Cref{alg:SPEA2} may seem different from the original \SPEA by \citet{ZitzlerLT01}, the two algorithms are in fact almost equivalent.
The terms \emph{archive} and \emph{population} from the original algorithm correspond respectively to the \emph{parent population} and the \emph{offspring population} in our formulation, which uses a more standard terminology for the field of runtime analysis. \Cref{tab:notation-comparison} \cite[Table~II]{WiethegerD24} summarizes the different notation used in the two formulations.

\begin{table}[t]
    \centering
    \begin{tabular}{l l}
        \toprule
        \textbf{Original SPEA2}
                                    & \textbf{Our SPEA2 formulation}      \\
        \midrule
        archive $\overline{P}_t$    & parent population $P_t$             \\
        archive size $\overline{N}$ & parent population size $\mu$        \\
        population $P_t$            & offspring population $Q_t$          \\
        population size $N$         & offspring population size $\lambda$ \\
        \bottomrule
    \end{tabular}
    \caption{Comparison of the original notation in \cite{ZitzlerLT01} and our notation.}
    \label{tab:notation-comparison}
\end{table}

Using our notation, we compare the two algorithms. The original \SPEA starts with an empty parent population and a random offspring population, whereas for our variant, it is vice versa. Furthermore, within one iteration of the original \SPEA, the selection for survival is performed before offspring generation, which is not a common order for evolutionary algorithms. Our formulation takes the more common order, that is, offspring generation is performed first. In the original \SPEA, when selection for survival is performed for the first time, there are $\lambda$ individuals to choose from (empty parent population and random offspring population), meaning that the algorithm is not properly defined for $\lambda < \mu$. In our variant, there are always $\mu + \lambda$ individuals to choose from when selecting for survival, hence it is well defined for all values $\lambda$ and $\mu$. The last difference between these two algorithms is in the way parents are selected for offspring generation. The original \SPEA uses binary tournaments, whereas our formulation selects parents uniformly at random. We let the reader check that, apart from the differences mentioned above, the two algorithms are equivalent.

\subsection{Benchmarks}
\label{sec:benchmarks}

We analyze the runtime of \SPEA for the \oneminmax, \ojzj, and \leadingonestrailingzeros benchmark functions. These three benchmarks are among the most established ones for theoretical runtime analysis of MOEAs, featured in almost all MOEA theory results. All functions are defined on the set of bit strings of length $n \in \Z_{\geq 1}$.

The \textbf{\oneminmax} (\textbf{\OMM}) benchmark, proposed in \cite{GielL10}, is a bi-objective function that returns the number of 0s and 1s in a bit string. Formally, we have
\begin{align*}
    \OMM \colon  \{0,1\}^n \rightarrow \R ^2 ,
    x                      \mapsto ( \zeros{x},  \ones{x}) .
\end{align*}
In this setting, every individual is Pareto-optimal, and the Pareto front is $\{ (k, n-k) \mid k \in [0..n]\} $.

The \textbf{\ojzj} (\textbf{OJZJ}) benchmark \cite{DoerrZ21aaai} is similar to \OMM in a way that both functions depend only on the total number of ones in a bit string. \ojzj, however, has an additional parameter $k \in [2..\floor{n/2}]$, called \emph{gap size}. Only the individuals that have between $k$ and $n-k$ ones, and the individuals $1^n$ and $0^n$ are Pareto-optimal, and they dominate all other individuals. Formally, for $n \in \Z_{\geq 2}$, $k \in [2..n]$, the \ojzj  benchmark with gap size $k$ (\ojzjk) is defined as
\begin{align*}
    \ojzjk \colon \{0,1\}^n \rightarrow \R^2,
    x                       \mapsto (f_1(x), f_0(x)),
\end{align*}
where, for $i \in \{0,1\}$, we have
\begin{align*}
    f_i(x) =
    \begin{cases}
        k + \abs{x}_i & \text{ if } \abs{x}_i \leq n - k \text{ or } \abs{x}_i = n, \\
        n - \abs{x}_i & \text{ otherwise. }
    \end{cases}
\end{align*}

The Pareto front is $F^* = \{ (i,n+2k - i) \mid i \in [2k..n] \cup \{k, n+k\}\} $, and the Pareto set is $S^* = \{ x \in \{0,1\}^n \mid \ones{x} \in [k..n-k] \cup \{0, n\}\} $ \cite[Theorem~5]{DoerrZ21aaai}. We define $F^*_{I} \coloneqq \{ (i,n+2k - i) \mid i \in [2k..n] \}$ as the inner region of the Pareto front, and $S_{I}^* \coloneqq \{ x \in \{0,1\}^n \mid \ones{x} \in [k..n-k] \}$ as the inner region of the Pareto set.
Last, the largest set of pairwise incomparable solutions has a cardinality of at most $n - 2k + 3$ \cite[Corollary~6]{DoerrZ21aaai}.

The \textbf{\leadingonestrailingzeros} (\textbf{\lotz}) \cite{LaumannsTZ04} is a bi-objective function
\begin{align*}
    f \colon \{0,1\}^n \rightarrow \R^2,
    x                  \mapsto  (f_1(x), f_2(x)),
\end{align*}
where $f_1$ returns the length of the longest prefix of $1s$, and~$f_2$ returns the longest suffix of $0s$. Formally,
\begin{align*}
    f_1(x)  = \sum_{i=1}^n \prod_{j=1}^ix_j \quad \text{and} \quad
    f_2(x)  = \sum_{i=1}^n \prod_{j=0}^{i-1}(1 - x_{n-j}).
\end{align*}

The Pareto front is $\{ (k, n-k) \mid k \in [0..n]\}$, which is also the largest set of incomparable values. The Pareto set is $\{ 1^k0^{n-k} \mid k \in [0..n] \}$. Contrary to \OMM and \ojzjk, in \lotz, we have a 1-to-1 correspondence between the Pareto set and the Pareto front. Hence, covering the Pareto front is equivalent to finding all Pareto optimal individuals.

\section{Population Dynamics of the \SPEA}
\label{sec:population-dynamics}

The main reason for the superior runtimes we prove for the \SPEA in this work are its different population dynamics.
As we will show in the following analysis, the selection mechanism of the \SPEA, different from the one of the classic \NSGA, removes from the set of non-dominated solutions always an individual with objective value most present in this set. This results in different objective values being present in an as balanced manner as possible. In particular, there is a tendency to increase the number of individuals with a recently found solution value, which in turn aides finding an unseen solution value from these.

The following lemma makes our finding on the more balanced selection of the \SPEA precise.

\begin{lemma}
    \label{lem:survival-guarantee}
    Consider the \SPEA with $\mu, \lambda \in \Z_{\geq 1}$ optimizing some multi-objective problem $f \colon  \Omega  \rightarrow \R ^m $, and consider an arbitrary iteration $t \in \Z_{\geq 0}$.
    As in Algorithm~\ref{alg:SPEA2}, let $S_t$ be the multi-set of non-dominated individuals from the total population $R_t$. Let $F_t = \{ f(x) \mid x \in S_t \}$ be the set of distinct objective values of individuals from $S_t$. Let $u \in \R^m$, and let $A_u = \{x \in S_t \mid f(x) = u \}$ be the multi-set of non-dominated individuals with objective value~$u$.\footnote{Clearly, $A_u$ is empty if there is no individual with objective value~$u$, or if such individuals are strictly dominated by others. It is nevertheless convenient to have also these trivial cases covered.}
    Then at least
    \[
        \min \{ \abs{A_u}, \floor{\tfrac{\mu}{\abs{F_t}}} \}
    \]
    individuals from $A_u$ survive to the next iteration.

    In particular, if $\overline M$ is the maximum size of an incomparable set of solutions of the problem~$f$, then at least
    \[
        \min \{ \abs{A_u}, \floor{\tfrac{\mu}{\phantom{.}\overline M\phantom{.}}} \}
    \]
    individuals from $A_u$ survive to the next iteration.
\end{lemma}

\begin{proof}
    There is nothing to show if $A_u = \emptyset$, so let us assume that $S_t$ contains individuals with objective value~$u$. If $\abs{S_t} \leq \mu$, all non-dominated individuals survive, and the result again follows trivially. Hence, we assume that $\abs{S_t} > \mu$.

    In iteration $t$, the algorithm eliminates $\abs{S_t} - \mu $ individuals from $S_t$, one by one, according to the $\sigma_k$ criterion. Let $i \in [0..\abs{S_t}-\mu]$ and assume that $i$ individuals have already been removed. Denote by $S_t^i$ and $A_u^i$ the multi-sets of the remaining individuals from $S_t$ and $A_u$, respectively. Note that $A_u^0 = A_u$, and $A_u^{\abs{S_t} - \mu}$ represents the individuals from $A_u$ that survive to iteration $t+1$.

    Our goal is to show that, if  $i \leq \abs{S_t} - \mu - 1$ and $\abs{A_u^i} \leq \floor{\frac{\mu}{\abs{F_t}}}$, then we do not remove an individual from $A_u^i$, i.e, we have $A_u^i = A_u^{i+1}$.  With an elementary induction, this implies that $A_u^{\abs{S_t} - \mu} \geq \min \{ \abs{A_u}, \floor{\frac{\mu}{\abs{F_t}}} \} $, and the result follows.

    Assume that $i \leq \abs{S_t} - \mu - 1$ and $\abs{A_u^i} \leq \floor{\frac{\mu}{\abs{F_t}}}$.
    Since $\{ A_v^i \}_{v \in F_t}$ is a partition of $S_t^i$ and $\abs{S_t^i} = \abs{S_t} - i > \mu$, by the pigeonhole principle, there exists $v \in F_t $ such that $\abs{A_v^i} > \floor{\frac{\mu}{\abs{F_t}}} \geq \abs{A_u^i} $.

    Now, we show that no element $y \in A_u^i$ is eliminated.  Let $x \in A_v^i$ and $y \in A_u^i$. Then, since $\abs{A_v^i} \geq  \abs{A_u^i} + 1$, for all $k \in [1..\abs{A_u^i} - 1]$, we have $\sigma_k(x) = \sigma_k(y) = 0$, but $\sigma_{\abs{A_u^i}}(x) =0 < \sigma_{\abs{A_u^i}}(y)$. Thus, we have $y \not\leq_d x$, meaning that $y$ is not minimal with respect to $\leq_d$. Therefore, $y$ is not removed. Since this holds for any $y \in A_u^i$, the set $A_u^i$ remains unchanged. This proves the main claim.

    The last claim follows immediately from the fact $|F_t| \le \overline M$. So see the latter, take any collection $S$ of individuals such that $S$ contains exactly one individual for each objective value in $F_t$. Then $S$ is an incomparable set and hence $|F_t| = |S| \le \overline M$.
\end{proof}

In the remainder, we restrict ourselves to the common case of bit string representations, that is, that the search space is $\Omega = \{0,1\}^n$. We note that comparable results could be shown for other search spaces.

We now build on the balanced selection property and show that the number of individuals with a given objective values exhibits a multiplicative growth behavior. In this analysis, we use the following
bound on the expected hitting time of a stochastic process that exhibits a certain kind of stochastic multiplicative growth. We note that the non-intuitive constant $0.29$ stems from the fact that we use an estimate on binomial distributions from~\cite{Doerr18exceedexp}.

\begin{lemma}
    \label{lem:mult-growth}
    Let $k \in \Z_{\geq 1}$ and $r \in (0,1)$ such that $kr \geq 0.29$. Consider two sequences of integer-valued random variables $(X_t)_{t \in \Z_{\geq 0}}$ and $(Y_t)_{t \in \Z_{\geq 1}}$ such that, for all $t \in \Z_{\geq 0}$, we have
    \begin{align*}
         & X_0 = 1,                                      \\
         & Y_{t+1} \succeq \Bin(k, \min \{ r X_t, 1 \}), \\
         & X_{t+1} = X_t + Y_{t+1}.
    \end{align*}

    Let $B \in [1, 1/r]$ and $T=\inf \{t \geq 0 \mid X_t \geq B \}$. Then
    \[
        \E[T] \leq 4 \ceil{\log_{1+kr}B}.
    \]
\end{lemma}

In order to prove \Cref{lem:mult-growth}, we use of the following result on the binomial distribution~\cite[Theorem~10]{Doerr18exceedexp}.

\begin{theorem}
    \label{thm:bin-expectation}
    Let $\alpha \coloneqq \ln\frac{4}{3} < 0.2877$. Let $n \in \Z_{\ge 0}$ and $\frac{\alpha}{n} \leq p < 1$. Let $X \sim \Bin(n, p)$. Then $\Pr[X > \E[X]] \geq \frac{1}{4}$.
\end{theorem}

\begin{proof}[Proof of \Cref{lem:mult-growth}]
    We first argue that we can assume without loss of generality that $Y_{t+1} \sim \Bin(k, \min \{ r X_t, 1 \})$ instead of the more general case $Y_{t+1} \succeq \Bin(k, \min \{ r X_t, 1 \})$. To see this, consider sequences $(X'_t)_{t \in \Z_{\geq 0}}$ and $(Y'_t)_{t \in \Z_{\geq 1}}$ such that, for all $t \in \Z_{\geq 0}$, we have
    \begin{align*}
         & X'_0 = 1,                                   \\
         & Y'_{t+1} \sim \Bin(k, \min \{r X'_t, 1 \}), \\
         & X'_{t+1} = X'_t + Y'_{t+1}.
    \end{align*}
    Then a simple induction shows that, for all $t \in \Z_{\geq 0}$, we have $X_t \succeq X'_t$. Hence any probabilistic upper bound for $T'=\inf\{t \geq 0 \mid X'_t \geq B \}$ immediately extends to an upper bound for $T$. To ease reading, we now assume that the sequences $(X_t)_{t \in \Z_{\geq 0}}$ and $(Y_t)_{t \in \Z_{\geq 1}}$ from the statement of this lemma satisfy $Y_{t+1} \sim \Bin(k, \min \{ r X_t, 1 \})$ for all $t \ge 0$.

    Let $b = \ceil{\log_{1+kr} B}$. For $i \in [0..b]$, define $B_i = \min\{B, (1+kr)^i\}$ and $T_i = \inf \{t \geq 0 \mid X_t \geq B_i \}$. Since $X_0 = 1$ and $B_0 =  1$, we have $T_0 = 0$. Moreover, we have $X_{T_b} \geq B_b = \min \{B, (1 + kr)^b \} =  B$, hence $T \leq T_b$. We now estimate $\E[T_b]$. Since $T_0 = 0$, we have
    \[
        \E[T_b] = \E[T_0] + \sum_{i=1}^b \E[T_i - T_{i-1}] = \sum_{i=1}^b \E[T_i - T_{i-1}] \tag{1} \label{eq:sum-waiting-times}.
    \]

    We show that, for all $i \in [1..b]$, we have $\E[T_i - T_{i-1}] \leq 4$. To this aim, let $i \in [1..b]$ and $t \in [T_{i-1}..T_i-1]$. Assume first that $X_t \geq B$. Then, since $(X_{t'})_{t' \in \Z_{\geq 0}}$ is non-decreasing, we have
    \begin{align*}
        \Pr[X_{t+1} \geq B_i]\geq \Pr[X_{t+1} \geq B] =1.
    \end{align*}

    Assume now that $X_t < B$. Since $t \geq T_{i-1}$, we have $X_t \geq (1+kr)^{i-1}$. Therefore,
    \begin{align*}
        \Pr[X_{t+1} \geq B_i] & \geq \Pr[X_{t +1}\geq (1+kr)^i]    \\
                              & \geq \Pr[X_{t+1} \geq X_t (1+kr) ] \\
                              & = \Pr[Y_{t+1} \geq k r X_t].
    \end{align*}

    Since $r X_t < rB \leq 1$, we have $Y_{t+1} \sim \Bin(k, r X_t)$. Furthermore, since $X_{t'} \geq 1$ for all $t' \in \Z_{\geq 0}$, and by the assumption that $kr \geq 0.29$, we have $0.29/k \leq r X_t < 1$. Hence, we apply \Cref{thm:bin-expectation} to $Y_{t+1}$, to obtain $\Pr[Y_{t+1}\geq k r X_t] \geq 1/4$. Therefore, for all $t \geq T_{i-1}$ we have
    \[
        \Pr[X_{t+1} \geq B_i] \geq 1/4,
    \]
    implying  $\E[T_i - T_{i-1}] \leq 4$, as desired.

    Going back to equation \eqref{eq:sum-waiting-times}, we obtain $\E[T_b] \leq 4b = 4 \ceil{\log_{1+kr}B}$. Thus, $\E[T] \leq \E[T_b] \leq 4 \ceil{\log_{1+kr}B}$ as desired.
\end{proof}

We now conduct the promised analysis of how the number of individuals with a given objective value grows.

\begin{lemma}
    \label{lem:mult-growth2}
    Consider an iteration $t_0 \in \Z_{\geq 0}$ of the \SPEA optimizing a function $f\colon \{0,1\}^n \rightarrow \R^m$, where $n \geq 3$. Let $\overline M$ be the size of the largest set of pairwise incomparable objective values of $f$, and assume that the parent population size is $\mu \geq \overline M$.

    Let $R_t$ and $F_t$ be as in \Cref{lem:survival-guarantee}.
    Consider a non-dominated objective value $v \in F_{t_0}$. For $t \in \Z_{\geq 0}$, define $X_t = \abs{ \{x \in R_{t_0 + t} \mid f(x) = v \} }$. Let $B \in [1..\floor{\frac{\mu}{\phantom{.}{\overline M}\phantom{.}}}]$, and let $T$ be the smallest integer~$t$ such that $X_t \geq B$ or~$v$ is strictly dominated by a $u \in F_{t_0+t}$. Then $\E[T] = O(\ceil{\frac{\mu}{\lambda}} \log B)$.
\end{lemma}

\begin{proof}
    Our goal is to apply \Cref{lem:mult-growth} to the process  $(X_t)_{t \in \Z_{\ge 0}}$ in order to obtain a bound on $\E[T]$. Thus, we carefully analyze the behavior of $(X_t)_{t \in \Z_{\ge 0}}$.

    Since $v \in F_{t_0}$, there exists $x \in R_{t_0}$ such that $f(x) = v$, hence $X_0 \geq 1$. Let $t \in [0..T-1]$. Note that, since all values in $F_{t_0+t}$ are non-dominated, they are also pairwise incomparable. Thus, we have $\abs{F_{t_0+t}} \leq {\overline M}$, and $X_t < B \leq  \floor{\frac{\mu}{\phantom{.}{\overline M}\phantom{.}}} \leq  \floor{\frac{\mu}{\abs{F_{t_0 + t}}}}$.
    Since $t < T$, all individuals $x \in R_{t_0+t}$ with $f(x) = v$ are non-dominated. Therefore, by \Cref{lem:survival-guarantee}, at least $\min \{B, X_t\} = X_t$ of them are  present in  $P_{t_0+t+1}$. All these individuals are also kept in $R_{t_0+t+1}$. Thus, $(X_t)_{t \in \Z_{\ge 0}} $ is non-decreasing.

    We now estimate $X_{t+1} - X_t$. In every iteration, the offspring population is generated by $\lambda$ times choosing a parent uniformly at random and applying standard bit mutation. Recall that there are at least $X_t$ individuals $x$ with $f(x) = v$ in the parent population $P_{t_0+t+1}$. Hence, when generating offspring in iteration $t_0 + t + 1$, the probability of choosing a parent with $f(x) = v$ is at least $\frac{X_t}{\mu}$. Thus, the probability of choosing such a parent and generating a copy of it as an offspring is at least $\frac{X_t}{\mu} (1 - \frac{1}{n})^n $. Let $\delta = 0.29$. Since $(1 - \frac{1}{n})^n$ is increasing in $n$ and, by assumption, $n \geq 3$, we have $(1 - \frac{1}{n})^n \geq \frac{8}{27} > 0.29$ and, therefore, $\frac{X_t}{\mu} (1 - \frac{1}{n})^n \geq \delta \frac{X_t}{\mu} $. Thus, we have $X_{t+1} - X_t \succeq \Bin ( \lambda,  \delta \frac{X_t}{\mu} )$.

    We observe that $(X_t)_{t \in \Z_{\ge 0}} $ exhibits multiplicative growth, similar to the one required by \Cref{lem:mult-growth}. However, we cannot apply the lemma because we cannot guarantee that $\delta \frac{\lambda}{\mu} \geq 0.29$. It is the factor of $\frac{\lambda}{\mu}$ in the previous inequality that is problematic. Thus, we consider the progress that $(X_t)_{t \in \Z_{\ge 0}}$ makes in $\ceil{\frac{\mu}{\lambda}}$ iterations, that is, we study the subsequence $(X_{t \ceil{\frac{\mu}{\lambda}}})_{t \in \Z_{\ge 0}}$ of $(X_t)_{t \in \Z_{\ge 0}}$.

    Let $t \in [0 ..  \floor{T  / \ceil{\frac{\mu}{\lambda}} } -1]$. We estimate  $X_{(t + 1)\ceil{\frac{\mu}{\lambda}}} - X_{t\ceil{\frac{\mu}{\lambda}}}$ using similar arguments as before. Let $s \in [t\ceil{\frac{\mu}{\lambda}}.. (t+1)\ceil{\frac{\mu}{\lambda}}-1]$. The number of individuals $x$ with $f(x) = v$ in the parent population $P_{t_0+s+1}$ is at least $X_s \geq X_{t\ceil{\frac{\mu}{\lambda}}}$. Hence, in iteration $t_0+s+1$, the probability of choosing a parent $x$ with $f(x) = v$ and generating a copy of it as an offspring is at least $\frac{X_{t \ceil{\mu / \lambda}}}{\mu} (1 - \frac{1}{n})^n \geq \delta \frac{X_{t \ceil{\mu / \lambda}}}{\mu} $. Since $\lambda$ offspring are generated per iteration, the total number of offspring generated in iterations $[t\ceil{\frac{\mu}{\lambda}}.. (t+1)\ceil{\frac{\mu}{\lambda}}-1]$ is $\lambda\ceil{\frac{\mu}{\lambda}} $. Therefore, we have $X_{(t + 1)\ceil{\frac{\mu}{\lambda}}} - X_{t\ceil{\frac{\mu}{\lambda}}} \succeq \Bin ( \lambda\ceil{\frac{\mu}{\lambda}} ,  \delta \frac{X_{t\ceil{\mu / \lambda} }}{\mu})$.

    Let $\widetilde{T}$ be the smallest $t \in \Z_{\ge 0}$ such that $X_{t\ceil{ \frac{\mu}{\lambda}} } \geq B$ or $v$ is strictly dominated by some  $u \in F_{t_0+t\ceil{ \frac{\mu}{\lambda}}}$. We apply \Cref{lem:mult-growth} to the sequence $(X_{t\ceil{\frac{\mu}{\lambda}}})_{t \in \Z_{\ge 0}}$ to bound $\E[\widetilde{T}]$.
    Note that the sequence  $(X_{t\ceil{\frac{\mu}{\lambda}}})_{t \in \Z_{\ge 0}}$ exhibits multiplicative growth required by \Cref{lem:mult-growth} before it reaches the bound $B$ or a value strictly better than $v$ is found. That is, for $t \in [0..\widetilde{T} - 2]$ we have $X_{(t + 1)\ceil{\frac{\mu}{\lambda}}} - X_{t\ceil{\frac{\mu}{\lambda}}} \succeq \Bin ( \lambda\ceil{\frac{\mu}{\lambda}} ,  \delta \frac{X_{t\ceil{\mu / \lambda} }}{\mu} )$. For $t \geq \widetilde{T} - 1$ we are not concerned by the distribution of   $X_{(t + 1)\ceil{\frac{\mu}{\lambda}}} - X_{t\ceil{\frac{\mu}{\lambda}}}$ because either $X_{(t+1) \ceil{ \frac{\mu}{\lambda}}}$ already reaches the bound $B$, or $F_{t_0+(t+1) \ceil{ \frac{\mu}{\lambda}}}$ contains a value strictly better than $v$.
    We verify the other  assumptions of \Cref{lem:mult-growth}. First, we have $ \lambda \ceil{\frac{\mu}{\lambda}}  \frac{\delta}{\mu} \geq \delta = 0.29$. Second, since ${\overline M} \geq 1 > \delta$, we also have $B \leq \floor{\frac{\mu}{\phantom{.}{\overline M}\phantom{.}}} \leq \frac{\mu}{\delta}$. Thus, by \Cref{lem:mult-growth}, we have $\E[\widetilde{T}] \leq  4 \ceil{\log_{1+  \lambda \ceil{\frac{\mu}{\lambda}}  \frac{\delta}{\mu}}B}$. Moreover, since $X_{\widetilde{T}\ceil{\mu / \lambda}} \geq B$ or $F_{t_0 + \widetilde{T} \ceil{ \frac{\mu}{\lambda}}}$ contains a value strictly better than $v$, we have $T \leq \widetilde{T}\ceil{\mu / \lambda}$. Therefore, $\E[T] \leq \E[\widetilde{T}] \ceil{\mu / \lambda} = O( \ceil{\mu / \lambda} \log B )$.
\end{proof}

\section{Runtime Analysis}
\label{sec:runtime-analysis}

We now conduct our mathematical runtime analyses of the \SPEA on the \oneminmax, \ojzj, and \lotz benchmarks. The main work for this has been done in the previous section. However, both to fully exploit the insights from the previous section and because of some particularities of the \SPEA, we still need a couple of new arguments compared to the existing runtime analyses of other MOEAs on these benchmarks.

\subsection{Runtime Analysis for the OneMinMax Problem}
\label{sec:omm-runtime}

For the \OMM problem, we prove in \Cref{thm:omm-runtime-bound} that when the parent population size of the \SPEA satisfies $\mu \geq n+1$, the expected runtime is $O(\frac{\mu n}{\lambda} + n + \frac{n^2}{\lambda} \log n)$ iterations, or $O(\mu n + \lambda n + n^2 \log n)$ function evaluations. Consequently, when the total population size $\lambda + \mu$ is $O(n \log n)$,
the algorithm takes $O(n^2 \log n)$ function evaluations in expectation with no further influence on the population sizes.

\begin{theorem}
    \label{thm:omm-runtime-bound}
    The expected runtime of the \SPEA with parent population size $\mu \geq n+1$ and offspring population size $\lambda$ optimizing the \OMM benchmark is $ O (\frac{\mu n}{\lambda} + n + \frac{n^2}{\lambda} \log n  )$ iterations, or $ O (  (\mu + \lambda )n + n^2 \log n  ) $ function evaluations.
\end{theorem}

In order to bound the expected runtime of \SPEA, we use an approach similar to the one of \cite{DoerrIK25} for the balanced \NSGA, but with two key differences. First, we do not use the multiplicative up-drift theorem \cite[Theorem 2.1]{DoerrK21algo} as it is not applicable in this scenario but instead use our multiplicative growth lemma (\Cref{lem:mult-growth}). The second difference is that we extend the bound to potentially different parent and offspring population sizes, which are common for the \SPEA but not for the (balanced) \NSGA.

The broad outline of our main proof (of \Cref{thm:omm-runtime-bound}) is to add up the waiting times for finding a new solution value on the Pareto front given that a neighboring value is already present. These times are estimated in Lemma~\ref{lem:new-value-discovery-time}, by first using Lemma~\ref{lem:mult-growth2} to analyze the time until a suitable number of individuals with the neighboring objective value are in the population and then estimating the time it takes to generate from one of these parents the desired new solution value. To obtain the best estimate for the runtime, the number of these parents we wait for depends on the difficulty of the mutation generating the desired offspring.

\begin{lemma}
    \label{lem:new-value-discovery-time}
    Let $n \ge 3$. Consider the \SPEA with parent population size $\mu \geq n+1$ and offspring population size $\lambda$, optimizing the \OMM benchmark.
    Let $v \in [1..n]$ and $i \in \{0, 1 \}$. Assume that in some iteration $t_0 \in \Z_{\geq 0}$ there exists an individual $x_0 \in P_{t_0}$  such that $\abs{x_0}_i = v$. Denote by~$T_v^i$ the number of iterations needed to generate an individual~$y$ with $\abs{y}_i = v -1$. Then
    \[
        \E[T_v^i] = O \left( \ceil*{\frac{\mu}{\lambda}} \log \ceil*{\frac{n}{v}} + \frac{\mu}{\lambda} + \frac{n^2}{\lambda v} + 1 \right).
    \]
\end{lemma}

\begin{proof}
    Let $B_v = \min \{ \ceil{\frac{n}{v}}, \floor{\frac{\mu}{n+1}} \}$. Then $B_v \geq 1$, since $\mu \geq n+1$ and $v \in [1..n]$. For $t \in \Z_{\ge 0}$, define $X_t = \abs{ \{ x \in R_{t_0 + t} \mid \abs{x}_i = v \} }$ as the number of individuals with $\abs{x}_i = v$ in the total population at iteration $t_0 + t$. Let $T' = \inf \{ t \geq 0 \mid X_t \geq B_v \}$ be the number of iterations needed to accumulate at least $B_v$ individuals $x$ with $\abs{x}_i=v$, and let $T''$ be the number of additional iterations needed to produce an individual $y$ such that $\abs{y}_i = v-1$, after accumulating at least $B_v$ individuals $x$ with $\abs{x}_i = v$. In other words, $T'' = \inf \{ t \geq 1 \mid \exists y \in Q_{t_0 + T' + t} \colon \abs{y}_i = v-1 \}$.

    Since by definition of~$T''$, there exists $y \in Q_{t_0 + T' + T''} \subset R_{t_0 + T' + T''}$ with $\abs{y}_i = v-1$, we have $T_v^i \leq T' + T''$. By linearity of expectation, we also have $\E[T_v^i] \leq \E[T'] + \E[T'']$. Hence, we now estimate $\E[T']$ and $\E[T'']$.

    First, we show that $\E[T'] = O(\ceil{\frac{\mu}{\lambda}} \log \ceil{\frac{n}{v}} )$.
    We apply \Cref{lem:mult-growth2} to the sequence $(X_t)$. For \OMM, the largest set of pairwise incomparable solutions is the Pareto front, hence, ${\overline M} = n+1$. Moreover, any objective value is Pareto-optimal, and thus non-dominated. Therefore, by \Cref{lem:mult-growth2}, we obtain the desired bound, that is $\E[T'] =O(\ceil{\frac{\mu}{\lambda}} \log B_v ) = O(\ceil{\frac{\mu}{\lambda}} \log \ceil{\frac{n}{v}} )$.

    Now we show that $\E[T''] = O(\frac{\mu}{\lambda} + \frac{n^2}{\lambda v} + 1)$. By definition of $T'$, we have $X_{T'} \geq B_v$. Thus, by \Cref{lem:survival-guarantee}, for all iterations $t \geq t_0 + T' + 1$, we have at least $B_v$ individuals with $\abs{x}_i = v$ in the parent population $P_t$. Consider some iteration $t \geq t_0 + T' +1$. The probability of selecting a parent $x$ with $\abs{x}_i = v$ and generating an offspring $y$ with $\abs{y}_i = v-1$ from $x$ is at least $\frac{B_v}{\mu} v \frac{1}{n} (1-\frac{1}{n})^{n-1} \geq \frac{B_v v}{\mu n e}$, since $ v \frac{1}{n} (1-\frac{1}{n})^{n-1}$ accounts only for the case where one of the $v$ bits of $x$ is flipped. Since this procedure is repeated $\lambda$ times per iteration, the probability of generating an individual $y$ with $\abs{y}_i = v - 1$ by the end of iteration $t$ is at least $1 - (1 - \frac{B_v v}{\mu n e})^{\lambda}$. By Corollary~1.4.3 of \cite{Doerr20bookchapter},  we have $ (1 - \frac{B_v v}{\mu n e})^{\lambda} \leq \frac{1}{1 +  \frac{ \lambda B_v v}{\mu n e}} = \frac{\mu n e}{\mu n e + \lambda B_v v}$. Thus, the probability of generating an individual $y$ with $\abs{y}_i = v - 1$ in iteration $t$ is at least $1 - \frac{\mu n e}{\mu n e + \lambda B_v v} = \frac{\lambda B_v v}{\mu n e + \lambda B_v v}$. Hence, $T''$ is dominated by a geometric distribution with success rate $\frac{\lambda B_v v}{\mu n e + \lambda B_v v}$, and this gives
    \begin{align*}
        \E[T''] & \leq \frac{\mu n e + \lambda B_v v}{ \lambda B_v v} = O \left( \frac{\mu n }{ \lambda B_v v} + 1 \right)          \\
                & = O \left( \frac{\mu n }{ \lambda v} \frac{1}{\min \{ \ceil{\frac{n}{v}}, \ceil{\frac{\mu}{n+1}} \} } + 1 \right) \\
                & = O \left( \frac{\mu n }{ \lambda v} \max \left\{ \frac{v}{n}, \frac{n+1}{\mu} \right\} + 1 \right)               \\
                & = O \left( \frac{\mu}{\lambda} + \frac{n^2}{\lambda v} + 1 \right).
    \end{align*}
    Since $\E[T_v^i] \leq \E[T'] + \E[T'']$, we obtain $\E[T_v^i] = O(\ceil{\frac{\mu}{\lambda}} \log \ceil{\frac{n}{v}} + \frac{\mu}{\lambda} + \frac{n^2}{\lambda v} + 1).$
\end{proof}

Our main results now follows from adding up the waiting times from the previous lemma in a suitable order, a standard argument in the analysis of \OMM, and simplifying this sum. As it will be profitable in the analysis of \ojzj, we show a more general result, discussing the time to cover a subinterval of the Pareto front.

\begin{theorem}
    \label{thm:omm-runtime-bound-intervals}
    Consider the \SPEA with parent population size $\mu \geq n+1$ and offspring population size $\lambda$, optimizing the \OMM benchmark. Let $\alpha \in [0..\floor*{\frac{n}{2}}]$, and let~$N_{\alpha}$ denote the number of iterations it takes to cover the subset $\{ (v,n -v) \}_{v \in [\alpha..n - \alpha]}$ of the Pareto front, that is, $N_{\alpha} = \inf \{t \in \Z_{\geq 0} \mid [\alpha..n - \alpha] \subset \{ \ones{x} \mid x \in P_t \} \}$. Assume there exists an individual $x_0 \in P_0$ such that $\ones{x} \in [\alpha..n - \alpha]$. Then, $\E[N_{\alpha}] = O( \frac{\mu n}{\lambda} + n + \frac{n^2}{\lambda} \log n )$.
\end{theorem}

\begin{proof}
    We first note that for $n \le 2$, we trivially have a constant expected runtime. This follows from the fact that Pareto-optimal solution values of the \OMM problem are not lost when $\mu \ge n+1$, see~\cite{RenBLQ24}, and that any mutation can generate any other solution with probability at least $n^{-n} \ge 2^{-2} = \frac 14$. So let us assume in the remainder that $n \ge 3$.

    Note that, for any individual $x$, we have $\ones{x} \in [\alpha..n - \alpha] $ if and only if $\zeros{x} \in [\alpha..n - \alpha] $. Let $k_0 = \zeros{x_0}$. Using notation from \Cref{lem:new-value-discovery-time}, the expected number of iterations to produce individuals containing objective values $ \{ (v, n - v)\}_{v \in [\alpha..k_0]}$ is at most $ \sum_{v = \alpha + 1}^{k_0} \E[T_v^0]$. Similarly,  the expected number of iterations to produce individuals containing objective values $ \{ (v, n - v)\}_{v \in [k_0 + 1.. n - \alpha]}$ is at most $ \sum_{v = \alpha + 1}^{n - k_0} \E[T_v^1]$. Thus, the expected number of iterations to cover the whole subset $\{ (v,n -v) \}_{v \in [\alpha..n - \alpha]} $ of the Pareto front is at most $\sum_{v = \alpha + 1}^{k_0} \E[T_v^0] +  \sum_{v = \alpha + 1}^{n - k_0} \E[T_v^1] \leq \sum_{v = \alpha + 1}^{n - \alpha}\E[T_v^0] + \E[T_v^1]$, that is, $\E[N_{\alpha}] \leq  \sum_{v = \alpha + 1}^{n - \alpha}\E[T_v^0] + \E[T_v^1]$.

    By \Cref{lem:new-value-discovery-time}, for $v \in [\alpha+1..n-\alpha]$, $i \in \{0,1\}$, we have $E[T_v^i] = O(\ceil{\frac{\mu}{\lambda}} \log \ceil{\frac{n}{v}} + \frac{\mu}{\lambda} + \frac{n^2}{\lambda v} + 1).$ Hence, we have
    \begin{align*}
        \E[N_{\alpha}] & =  \sum_{v = \alpha + 1}^{n - \alpha} O \left( \ceil*{\frac{\mu}{\lambda}} \log \ceil*{\frac{n}{v}} + \frac{\mu}{\lambda} + \frac{n^2}{\lambda v} + 1 \right)         \\
                       & = O \left(  \frac{\mu n}{\lambda} + n +  \ceil*{\frac{\mu}{\lambda}} \sum_{v = 1}^{n} \log \frac{n+v}{v} + \frac{n^2}{\lambda } \sum_{v = 1}^{n} \frac{1}{v}  \right) \\
                       & = O \left(\frac{\mu n}{\lambda} + n +  \ceil*{\frac{\mu}{\lambda}} \log \frac{(2n)!}{(n!)^2} + \frac{n^2}{\lambda} \log n  \right)                                    \\
                       & =  O \left(\frac{\mu n}{\lambda} + n +  \ceil*{\frac{\mu}{\lambda}} \log 4^n + \frac{n^2}{\lambda} \log n  \right) \tag{1}  \label{eq:fact-to-exp}                    \\
                       & =  O \left(\frac{\mu n}{\lambda} + n + \frac{n^2}{\lambda} \log n  \right).
    \end{align*}
    To obtain equality \eqref{eq:fact-to-exp}, we used the fact that $\frac{(2n)!}{(n!)^2} \leq 4^n$.
\end{proof}

Applying \Cref{thm:omm-runtime-bound-intervals} with $\alpha = 0$ (for this we note that, trivially, the random initial population contains individuals on the Pareto front) we obtain the claimed runtime bound for the \SPEA applied to \OMM (\Cref{thm:omm-runtime-bound}).

\subsection{Runtime Analysis on the \ojzj Problem}
\label{sec:ojzj-runtime}

We move on to the \ojzj benchmark and show in \Cref{thm:ojzj-runtime} that for all gap sizes $k \in [2.. \floor{n/2}]$, the \SPEA computes the Pareto front in an expected runtime of $O(\frac{n^{k+1}}{\lambda} +  \frac{\mu}{\lambda} n  +  n)$ iterations, or $O (n^{k+1} + \mu n + \lambda n)$ function evaluations, when the parent population size is at least the size of the Pareto front, that is, $\mu \geq n -2k+3$. Consequently, unless a population size exceeds the asymptotic order of $n^k$, the runtime bound in function evaluations is (asymptotically) independent of the population sizes.

Until the end of the section, we consider the \SPEA with general population sizes $\mu$ and $\lambda$, optimizing the \ojzjk with gap size $k \in [2.. \floor{n/2}]$. Note that this implies $n \ge 4$.

\begin{theorem}
    \label{thm:ojzj-runtime}
    The expected runtime of the \SPEA with parent population size $\mu \geq n - 2k + 3$ and offspring population size $\lambda$ on the \ojzjk problem with gap size $k \in [2..\floor{n/2}]$ is $O(\frac{n^{k+1}}{\lambda} + \frac{\mu}{\lambda} n + n)$ iterations, or $O( n^{k+1} + \mu n + \lambda n)$ function evaluations.
\end{theorem}

We shall several times use the following elementary estimate to translate success probabilities of single mutations into runtime estimates.

\begin{lemma}
    \label{lem:expected-number-of-iterations-based-on-single-success}
    Let $\lambda \in \Z_{\geq 1}$.
    Consider a sequence of independent random Bernoulli trials, each with success probability $q \in (0, 1]$.
    We call~$\lambda$ non-overlapping consecutive trials an \emph{iteration}.
    Let~$T$ denote the first iteration containing a success.
    Then $\E[T] \leq 1 + \frac{1}{q \lambda}$.
\end{lemma}

\begin{proof}
    Let~$q'$ denote the probability that a single iteration contains a success.
    Since the trials are independent,~$T$ follows a geometric distribution with success probability~$q'$, resulting in $\E[T] = \frac{1}{q'}$.
    Hence, we estimate in the following a lower bound for~$q'$, which then proves the claim.

    Since the trials are independent, the probability of an iteration containing no success is $(1 - q)^\lambda = 1 - q'$.
    By \cite[Corollary~1.4.3]{Doerr20bookchapter}, we obtain $(1 - q)^\lambda \leq \frac{1}{1 + q \lambda}$.
    Consequently, we have $q'  \geq 1 - \frac{1}{1 + q \lambda} = \frac{q \lambda}{1 + q \lambda}$, which concludes the proof, noticing that $q \lambda > 0$.
\end{proof}

As in \cite{DoerrQ23tec}, we decompose the optimization process into three stages, which are analyzed separately in the following. The first estimates the time to find any solution on the inner part of the Pareto front. Since here any parent can be used (in our estimate), we do not require yet the main machinery of this work.

\begin{lemma}
    \label{lem:stage1}
    Let $T_1$ be the number of iterations needed to find the first inner Pareto optimum, that is $T_1 = \inf \{ t \in \Z_{\geq 0} \mid P_t \cap S_{I}^* \neq \emptyset \}$. Then $\E[T_1] \leq \frac{e}{\lambda} k^k + 1$.
\end{lemma}

\begin{proof}
    If the initial population $P_0$ contains an individual $x$ with $x \in S_{I}^*$, then trivially we have $T_1 = 0$.
    Assume now that all initial individuals are not on the inner part of the Pareto front. Hence for all $x \in P_0$, we have $\ones{x} < k$ or $\zeros{x} < k$. Consider a single mutation operation with parent individual~$y$. Without loss of generality, suppose that $\ell \coloneqq \ones{y} < k$. Let $y'$ be the offspring generated in this mutation operation. Then
    \begin{align*}
        \Pr & [y' \in S_{I}^* ] \geq \Pr[ \ones{y'} = k]                                                                \\
            & \geq \binom{n-\ell}{k-\ell} \left( \frac{1}{n} \right)^{k-\ell} \left( 1 - \frac{1}{n} \right)^{n-k+\ell} \\
            & \geq \frac{n^{\ell} (n-\ell)!}{(k-\ell)!(n-k)!} \frac{1}{n^k e}
        \geq \binom{n}{k} \frac{1}{n^ke}
        \geq \frac{1}{e k^k},
    \end{align*}
    where, in the last inequality, we use the fact that $\binom{n}{k} \geq ( \frac{n}{k} )^k$.

    The result follows by an application of \Cref{lem:expected-number-of-iterations-based-on-single-success}, as the mutations are independent.
\end{proof}

From now on, we assume that the Parent population size is at least $\mu \ge n - 2k + 3$.

The next stage is to find the whole inner Pareto front. This process is essentially identical to the one of finding a corresponding segment of the Pareto front of the \OMM problem. With foresight, we have proven in \Cref{thm:omm-runtime-bound} in the \OMM analysis a time bound for covering a subinterval of the Pareto front. Taking $\alpha = k$ and using exactly the same arguments, we obtain the following lemma.

\begin{lemma}
    \label{lem:stage2}
    Let $T_2$ be the number of iterations it takes to find the inner Pareto front once an inner Pareto optimum is in the population, that is,
    \[
        T_2 = \inf \{t \in \Z_{\geq 0} \mid F_{I}^* \subseteq f(P_t) \} - T_1.
    \]
    Then $\E[T_2]  = O( \frac{\mu n}{\lambda} + n + \frac{n^2}{\lambda} \log n)$.
\end{lemma}

It remains to find the two extremal points of the Pareto front. This analysis again builds heavily on our understanding of the population dynamics.

\begin{lemma}
    \label{lem:stage3}
    Let $T_3$ be the number of iterations the algorithm takes to find the two extremal Pareto optima once the inner part of the Pareto front is found, that is,
    \[
        T_3 = \inf \{t \in \Z_{\geq 0} \mid F^* \subseteq f(P_t) \} - T_2 - T_1.
    \]
    Then, $\E[T_3]  = O(  1 + \frac{n^{k+1}}{\lambda} + \ceil*{\frac{\mu}{\lambda}}\log \frac{\mu}{n - 2k + 3}  )$.
\end{lemma}

\begin{proof}
    To estimate $T_3$, we regard the time it takes to generate $1^n$ from an individual $x$ with $\ones{x} = n-k$, and $0^n$ from an $x$  with $\ones{x} = k$. By symmetry, we focus on the latter.

    We first regard the time it takes to have $B_k = \floor{ \frac{\mu}{n-2k+3}}$ individuals with $k$ ones in the population, aiming at applying \Cref{lem:mult-growth2}.

    For $t \in \Z_{\ge 0}$, define $X_t$ as the number of individuals with $\ones{x} = k$ in the $t$-th iteration of stage 3, that is $X_t = \abs{ \{ x \in R_{T_1 + T_2 + t} \mid \ones{x} = k\} }$. Since stage $3$ is entered once the inner region of the Pareto front is fully covered, we have $X_0 \geq 1$.
    Let $T' = \inf \{ t \in \Z_{\ge 0} \mid X_t \geq B_k \}$ be the number of iterations until $B_k$ individuals with $\ones{x} =k$ are accumulated. As mentioned in the preliminaries, the size of the largest set of pairwise incomparable values is $n-2k+3$. Moreover, the value $(2k, n)$ is on the Pareto front, hence non-dominated.
    Thus, by \Cref{lem:survival-guarantee}, we do not lose individuals with~$k$ as long as we have fewer than~$B_k$.
    Consequently, \Cref{lem:mult-growth2} yields  $\E[T'] = O( \ceil*{\frac{\mu}{\lambda}}\log B_k )$.

    Once we have $B_k$ individuals with $k$ ones, each newly generated individual has a probability of at least $\frac{B_k}{\mu} (\frac{1}{n})^k (1 - \frac{1}{n})^{n-k} \ge \frac{B_k}{\mu e n^k}$ of being the the desired solution $0^n$.
    By \Cref{lem:expected-number-of-iterations-based-on-single-success}, this results in an expected number of at most $\frac{\mu e n^k}{\lambda B_k} + 1 = O(\frac{n^{k+1}}{\lambda}) +1$ iterations.

    Together with the time to have $B_k$ possible parents, the expected time to have $0^n$ in the population is $O(1 + \frac{n^{k+1}}{\lambda} + \ceil{\frac{\mu}{\lambda}}\log \frac{\mu}{n - 2k + 3})$. Adding the identical estimate for the time to find $1^n$, we have proven the claim.
\end{proof}

By adding up the runtime estimates for each stage (and some not very interesting technical argument needed to simplify the runtime expression), we obtain the main result of this subsection.

\begin{proof}[Proof of \Cref{thm:ojzj-runtime}]
    We note that $E[T_1]$ (\Cref{lem:stage1}) is asymptotically smaller than our estimate for $E[T_3]$ (\Cref{lem:stage3}), which contains the term $\frac{n^{k+1}}{\lambda}$. Hence, also using \Cref{lem:stage2}, the expected runtime is at most
    \begin{align*}
        \E[T_2] + \E[T_3] & = O \left( \frac{n^{k+1}}{\lambda} + \frac{\mu}{\lambda} n + n\right.                 \\
                          & \qquad\quad \left. + \ceil*{\frac{\mu}{\lambda} } \log \frac{\mu}{n - 2k + 3} \right)
    \end{align*}
    iterations. If $\mu \leq 4^n$, then $\log \frac{\mu}{n} = O (n)$, and the last term is asymptotically not larger than the second or third, showing our claim. Assume now that $\mu > 4^n$. Then, since the parent population size is large, the initial population $P_0$ is likely to cover all possible bit strings $x  \in \{0,1\}^n$, meaning that the Pareto front is covered after 0 iterations. Indeed, the probability that $P_0$ does not contain a particular bit string $x \in \{0,1\}^n$ is
    \[
        (1 - 2^{-n})^{\mu} \leq e^{-2^{-n} \mu}
        \leq e^{- \mu^{1/2}}.
    \]
    Thus, by the union bound, the probability that there exists $x \in \{0,1\}^n$ such that $x \notin P_0$ is at most $2^n e^{- \mu^{1/2}}$. Only in this case, the algorithm needs more than $0$ iterations to cover the Pareto front. Since our previous analyses did not make any assumptions on the initial population, we can also for this case estimate the remaining runtime by $O(\frac{n^{k+1}}{\lambda} + \frac{\mu}{\lambda} n + n + \ceil{\frac{\mu}{\lambda} } \log \frac{\mu}{n})$.
    Therefore, if $\mu \geq 4^n$, the expected runtime of the algorithm is
    \begin{align*}
        2^n e^{- \mu^{1/2}} O \left( \frac{n^{k+1}}{\lambda} + \frac{\mu}{\lambda} n + n + \ceil*{\frac{\mu}{\lambda} } \log \frac{\mu}{n} \right)
        = O(1)
    \end{align*}
    iterations.
\end{proof}

\subsection{Runtime Analysis on the \leadingonestrailingzeros Problem}

We finally analyze in \Cref{thm:lotz-runtime} how the \SPEA optimizes the \lotz benchmark, the third very popular benchmark in the field. We prove a runtime guarantee of an expected number of $O( (\frac{\mu}{\lambda} + 1) n \log \frac{\mu}{n+1} +  \frac{n^3}{\lambda} + n)$ iterations, or ${O ((\mu + \lambda) n \log \frac{\mu}{n+1} +  n^3 + \lambda n)}$ function evaluations, provided that the parent population size $\mu$ is at least the size of the Pareto front $n+1$.

\begin{theorem}
    \label{thm:lotz-runtime}
    The expected runtime of the \SPEA with parent population size $\mu \geq n  + 1 $ and offspring population size $\lambda$, optimizing the \lotz benchmark is $O((\frac{\mu}{\lambda} + 1) n \log \frac{\mu}{n+1} +  \frac{n^3}{\lambda} + n)$ iterations, or ${O((\mu + \lambda) n \log \frac{\mu}{n+1} + n^3 + \lambda n)}$ function evaluations.
\end{theorem}

We follow the usual and natural approach of analyzing separately the times to first reach the Pareto front and then to spread out there until the full Pareto front is covered~\cite{LaumannsTZ04}. We cannot use the technically easier appearing variant of \cite{DoerrIK25} that first regards the growth of the number of leading ones in the population and then reuses this analysis for the number of trailing zeros. The reason is that an objective value corresponding to a certain number of leading ones may be dominated by another one with same number of leading ones, resetting the multiplicative growth process that is the heart of all of our results. We therefore regard how the maximum sum of objective values in the population grows. This avoids the problem since now a domination implies an actual increase of this progress measure. The time to increase the sum of objectives is estimated in the following lemma.

\begin{lemma}
    \label{lem:lotz-increase}
    For a set of individuals $P$, let $\nu(P) = \max\{f_1(x)+f_2(x) \mid x \in P\}$. Consider the \SPEA with $\mu \ge n+1$ optimizing the \lotz benchmark with $n \ge 3$. Assume that in some iteration $t_0 \in \Z_{\geq 0}$ we have $\nu(P_{t_0}) < n$.

    Let $T = \inf \{ t \in \Z_{> t_0} \mid \nu(P_t) > \nu(P_{t_0})\}$. Then
    \[
        \E[T - t_0] = O(\ceil{\tfrac{\mu}{\lambda}} \log \tfrac{\mu}{n+1} +  \tfrac{n^2}{\lambda} + 1).
    \]
\end{lemma}

\begin{proof}
    Let $v \in \R^2$ be an objective value witnessing $\nu(P_{t_0})$, that is, $P$ contains an individual $x$ with $f(x) = v$ and we have $v_1 + v_2 = \nu(P_{t_0})$. As in most other proofs in this work, we first exploit the multiplicative growth of the number of individuals with objective value $v$. Note that now, different from the previous analyses, it may happen that the objective value $v$ is dominated. However, in this case the dominating value witnesses a strictly larger $\nu$-value, so we are already done.

    With $B = \floor{\frac{\mu}{n+1}}$, we obtain from \Cref{lem:mult-growth2} that it takes an expected number of $O(\ceil{\frac{\mu}{\lambda}} \log B) = O(\ceil{\frac{\mu}{\lambda}} \log \frac{\mu}{n+1})$ iterations to have in the population at least $B$ individuals with objective value $v$ or to have at least one individual with objective value strictly dominating~$v$.

    In the latter case, we have increased the $\nu$-value of the population and we are done. Hence assume that the former case holds, that is, there are at least $B$ individuals with objective value $v$ in the population.

    By \Cref{lem:survival-guarantee}, in all future iterations until we have an individual dominating $v$, we have at least $B$ individuals with objective value~$v$. Hence in each such iteration the probability of selecting such an individual and mutating it into one with objective value strictly dominating~$v$ is at least $\frac{B}{\mu} \frac 1n (1-\frac 1n)^{n-1} \ge \frac{B}{\mu n e}$.

    By \Cref{lem:expected-number-of-iterations-based-on-single-success}, the expected number of iterations for such an individual to enter the population is
    \begin{align*}
        1 + \frac{\mu n e}{\lambda B}
        = O\left(\frac{\mu n}{\lambda } \frac{1}{\floor{\frac{\mu}{n+1}}} + 1\right) = O\left(\frac{n^2}{\lambda} + 1\right).
    \end{align*}
    Together with the time of the multiplicative growth phase,  we obtain an expected number of iterations of
    \[        O\left(\ceil*{\frac{\mu}{\lambda}} \log \frac{\mu}{n+1} + \frac{n^2}{\lambda} + 1\right).\qedhere
    \]
\end{proof}

The second ingredient to our main proof is a statement on how long it takes to generate a neighboring Pareto optimal objective value from an existing Pareto optimal value. The arguments for this are essentially the same as in the previous lemma, in fact, easier, since now we go from a fixed value $v$ to a fixed value $w$ and we know that $v$ cannot be lost as it is Pareto optimal. For that reason, we state the result without explicit proof.

\begin{lemma}
    \label{lem:lotz-spreadout}
    Consider the \SPEA with $\mu \ge n+1$ optimizing the \lotz benchmark with $n \ge 3$. Assume that in some iteration $t_0 \in \Z_{\geq 0}$, the population $P_{t_0}$ contains an individual with Pareto optimal objective value $v$. Let $w$ be a neighboring Pareto optimal solution value, that is, $w = (v_1+1,v_2-1)$ or $w=(v_1-1,v_2+1)$ provided this lies in $[0..n]^2$.

    Let $T = \inf \{t \in \Z_{\ge t_0} \mid w \in f(P_t)\}$. Then
    \[
        \E[T - t_0] = O(\ceil{\tfrac{\mu}{\lambda}} \log \tfrac{\mu}{n+1} +  \tfrac{n^2}{\lambda} + 1).
    \]
\end{lemma}

Using the two lemmas above, we obtain \Cref{thm:lotz-runtime} simply by adding up suitable waiting times.

\begin{proof}[Proof of \Cref{thm:lotz-runtime}]
    As in the proof of Theorem~\ref{thm:omm-runtime-bound}, we can assume that $n \ge 3$. By Lemma~\ref{lem:lotz-increase}, it takes an expected number of $O(\ceil{\tfrac{\mu}{\lambda}} \log \tfrac{\mu}{n+1} +  \tfrac{n^2}{\lambda} + 1)$ iterations to increase the maximum $f_1+f_2$ value in the population. Hence after an expected number of $O(n \ceil{\tfrac{\mu}{\lambda}} \log \tfrac{\mu}{n+1} +  \tfrac{n^3}{\lambda} + n)$ iterations, we have found a solution $x$ with $f_1(x)+f_2(x)=n$, that is, on the Pareto front. From this point on, we exploit Lemma~\ref{lem:lotz-spreadout}, telling us that we can find any neighbor of a Pareto optimal solution in the population in $O(\ceil{\tfrac{\mu}{\lambda}} \log \tfrac{\mu}{n+1} +  \tfrac{n^2}{\lambda} + 1)$ iterations. This needs to be done at most $n$ times, yielding a total time of $O(n \ceil{\tfrac{\mu}{\lambda}} \log \tfrac{\mu}{n+1} +  \tfrac{n^3}{\lambda} + n)$ iterations for completing the Pareto front.
\end{proof}

\section{Conclusion}

We conducted a rigorous runtime analysis of the \SPEA on the three main benchmarks used in the runtime analysis community. Our results show that, in particular in comparison with the well-studied \NSGA, the performance of the standard \SPEA is less sensitive to the choice of population size.
Combined with our more general main arguments, these findings suggest that the \SPEA can serve as a robust choice in practical applications where optimal population size is unknown.

We note that the selection in the \SPEA based on the $\sigma$-criterion is computationally more costly than the crowding distance employed by \NSGA.
However, the crowding distance combined with uniform tie-breaking can fail already for easy problems with at least three objectives \cite{ZhengD22arxivmany}, and the alternative proposed by \citet{DoerrIK25} adds complexity.
A theoretical comparison of such added costs is an interesting step for future work.

Last, in the case of a population size smaller than the Pareto front size, our results do not apply.
However, \citet{AlghouassDKL25} proved that the \SPEA converges within reasonable time to a uniform spread across the Pareto front in the \oneminmax benchmark, whereas the \NSGA does not.
An interesting next step is to see whether our analysis of the population dynamics also results in improved \SPEA convergence times for this setting.

\section*{Acknowledgments}
This research benefited from the support of the FMJH Program PGMO. This work has profited from many scientific discussions at the Dagstuhl Seminars 23361 ``Multiobjective Optimization on a Budget'' and 24271 ``Theory of Randomized Optimization Heuristics''.

}

\bibliography{ich_master.bib, alles_ea_master.bib, rest.bib}

\end{document}